\documentclass{article}

\usepackage[preprint]{neurips_2026}

\usepackage[utf8]{inputenc} 
\usepackage[T1]{fontenc}    
\usepackage{hyperref}       
\usepackage{url}            
\usepackage{booktabs}       
\usepackage{amsfonts}       
\usepackage{nicefrac}       
\usepackage{microtype}      
\usepackage{xcolor}         

\usepackage{amsmath}
\usepackage{amssymb}
\usepackage{amsthm}
\usepackage{algorithmic}
\usepackage{algorithm}
\usepackage{bauer, bauer_e}

\usepackage{color, colortbl}

\definecolor{LightCyan}{rgb}{0.88,1,1}
\definecolor{ColumnColor}{rgb}{1.0,0.941454,0.955287}

\newtheorem{inductivebias}{Inductive Bias}

\usepackage{graphicx}

\title{Rethinking Structural Anomaly Detection:\\ From Decision Boundaries to Projection Operators}

\author{%
  Alexander Bauer$^{1,2}$\\
  $^{1}$Machine Learning Group, TU Berlin\\
  $^{2}$BIFOLD, Berlin, Germany\\
  {\tt\small alexander.bauer@tu-berlin.de} \\
}

\begin{document}

\maketitle

\begin{abstract}
Most existing anomaly detection methods rely on estimating a probability density or learning
an enclosing decision boundary, implicitly assuming that normal data occupies a region of
non-zero volume in the ambient space. In contrast, structural anomaly detection considers data
that lies near a low-dimensional manifold, creating a mismatch between the inductive bias of existing
methods and the structure of the data, often resulting in degraded performance.
To address this mismatch, we introduce a geometric perspective.
Specifically, we learn a projection operator onto the manifold of normal samples and define
a sample as anomalous if it is altered by this projection.
This formulation naturally integrates the inductive bias of manifold-supported data and reframes anomaly detection in terms
of a projection residual, thereby resolving issues arising from modeling degenerate distributions.
Notably, it provides a unifying interpretation of reconstruction-based methods by explaining
their success and failure in terms of projection quality.
In particular, it explains the strong generalization ability of projection-aligned models as a consequence of contraction behavior toward the manifold.
Moreover, by decoupling anomaly detection from probabilistic modeling, it reduces the tendency
to misclassify rare but normal samples, a widely recognized limitation of existing approaches.
Empirically, we demonstrate that projection-aligned methods achieve strong performance,
outperforming boundary-based methods while improving upon existing reconstruction-based approaches.
\end{abstract}

\section{Introduction}
\label{sec:section1}

Conceptually, anomaly detection (AD) aims to identify observations that deviate
from a notion of normality.
In many practical scenarios, such as industrial inspection or medical imaging,
anomalies arise as structural deviations in input images, including surface defects,
shape deformations, or other irregularities disrupting normal visual patterns.

The vast majority of existing AD methods \cite{ScholkopfPSSW01, TaxD04, Parzen62, KimKYC23, RuffGDSVBMK18, TackMJS20,  RuffKVMSKDM21}
define normality either through a probability density,
or more generally, a decision boundary enclosing normal samples. 
Crucially, this relies on an implicit assumption that normal data occupies a region
of non-zero volume in the ambient space.
However, this assumption is fundamentally misaligned with high-dimensional perceptual data
such as images, where normal samples concentrate near a low-dimensional manifold.
Since an embedded manifold has no interior with respect to the ambient space,
any enclosing decision boundary necessarily includes unsupported regions,
rendering it conceptually ill-posed.
Employing boundary-based methods in this setting
leads to fundamental practical issues, including unstable learning behavior, sensitivity to noise
in the scoring function, and severe effects of the curse of dimensionality.
Figure~\ref{fig_bauer1} illustrates, using the example of a One-Class Support Vector Machine (OC-SVM) \cite{ScholkopfPSSW01}, how the mismatch between
the volumetric assumption of boundary-based methods and the low-dimensional nature
of manifold-supported data manifests geometrically.
When combined with isotropic radial kernels (e.g., a Gaussian kernel), similarity is propagated
equally in both tangent and normal directions of the data manifold.
Consequently, the kernel width induces an unavoidable trade-off:
large values admit larger off-manifold regions, leading to false negatives,
while small values require dense sampling to ensure manifold coverage and otherwise lead to false positives.
Although real-world data exhibits a non-zero volume through measurement error and finite-resolution effects,
these introduce only a thin neighborhood around the underlying manifold and do not alter its intrinsic structure in practice.

\begin{figure*}[t]
\centering
\includegraphics[scale = 0.724]{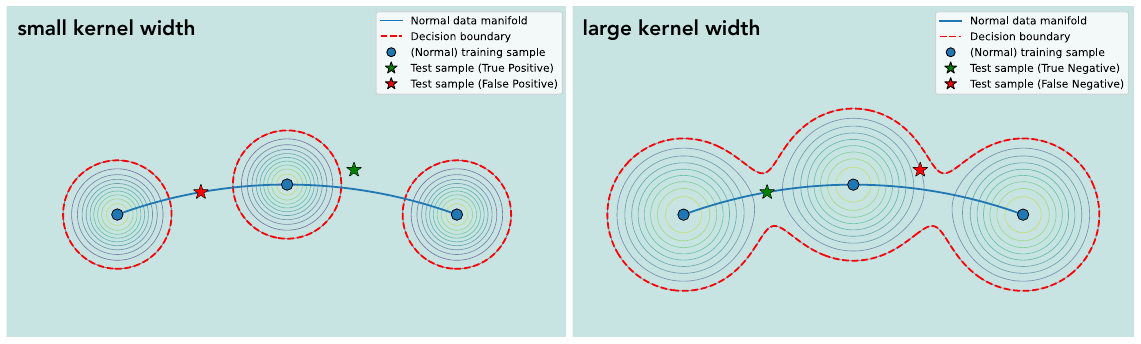}
\caption{
Illustration of the failure mode of an OC-SVM with an RBF kernel on manifold-supported data. Although normal samples concentrate near a low-dimensional manifold (blue curve), the isotropic similarity induced by the kernel in ambient space produces a volumetric inlier region. The effective radius of this region is controlled by the kernel width: decreasing it reduces coverage of the manifold and increases false positives (left), whereas increasing it expands the inlier region and increases false negatives (right). Note the change in classification of the star-shaped samples.
}
\label{fig_bauer1}
\end{figure*}

As the concept of an enclosing decision boundary becomes problematic for manifold-supported data,
a natural alternative is to measure the abnormality of a sample via its distance to the manifold.
However, this idea immediately encounters fundamental challenges:
the manifold itself is unknown and must be inferred from finitely many samples,
and computing distances requires solving a non-trivial optimization problem
to identify the closest point on the manifold.
In high-dimensional settings, both aspects are computationally and statistically intractable.
Instead, we propose a different approach.
The key observation is that, for natural images, the global shape of the data manifold
is induced by strong spatial correlations between pixels,
which drastically reduce the effective degrees of freedom.
Specifically, these dependencies give rise to a form of \emph{index-induced regularity},
rendering the value of a pixel largely predictable from its local neighborhood.
This enables implicit modeling of the global geometry from individual
samples by learning a mapping that corrects deviations from the manifold.
When applied to anomalous data, this mapping acts as an approximate projection operator,
mapping inputs toward the manifold without requiring explicit geometric representation.
In this way, global geometry becomes tied to local predictive structure shared across training samples, supporting the learning and generalization of a projection mapping.
The learned projection, in turn, induces an implicit characterization of the data manifold as the zero-set of a reconstruction functional.
This provides a global description of normality without requiring explicit geometric estimation or costly optimization to identify nearest points on the manifold.

To this end, we reinterpret and extend the family of correction-based methods~\cite{ZavrtanikKS21, LiSYP21, ZavrtanikKS212, InTra, math12243988}
for \emph{structural anomaly detection} (SAD) by viewing them as approximations of a learned projection operator onto the manifold of normal data.
Given an input, the model produces a corrected version projected toward the manifold, while abnormality is quantified through the resulting projection residual.
At the pixel level, the discrepancy between input and projection enables localization of anomalous regions.
Importantly, this perspective provides a unified framework for understanding and improving existing approaches.
First, it aligns the learning objective with the intrinsic geometry of the data, introducing an inductive bias consistent with manifold-supported distributions.
Second, it offers a principled explanation for the limitations of decision-boundary methods and the success of reconstruction-based approaches in terms of projection quality.
Third, by decoupling anomaly detection from probabilistic modeling, it reduces the tendency to misclassify rare but normal samples.
Fourth, it provides a geometric interpretation of generalization through contractive behavior toward the manifold, where diverse perturbations map to consistent normal representations.
Collectively, these insights provide a roadmap for improved algorithm design.
In particular, they motivate geometry-aware regularization through Jacobian constraints,
iterative refinement via fixed-point dynamics, and improved corruption strategies
that encourage stable projection behavior and stronger generalization.

The remainder of the paper is organized as follows.
Section~\ref{sec:section3} summarizes the training framework used to approximate the projection operator, including the inference procedure during prediction.
Section~\ref{sec:section4} presents a formal analysis of this framework by highlighting the conservative projection as an optimal solution to SAD and discussing how such a mapping is approximated through training.
We further examine how the projection-based perspective provides a geometric interpretation of model generalization and outline directions for future algorithmic improvements.
Section~\ref{sec:section5} shows the experimental evaluation on two established industrial anomaly detection benchmarks and compares the presented framework with state-of-the-art approaches.
Finally, Section~\ref{sec:section6} concludes the paper.

\section{Methodology}
\label{sec:section3}

Formally, we approximate the nonlinear projection operator using an autoencoder
\[
f_{\boldsymbol{\theta}} : [0,1]^{h \times w \times 3} \rightarrow [0,1]^{h \times w \times 3},
\]
where $\boldsymbol{\theta}$ denotes the learnable parameters.
Both input and output correspond to RGB images of resolution $h \times w$.
In particular, we make no architectural assumptions. Specifically, the model is
not required to contain a bottleneck.
The term autoencoder is used in a general sense, referring simply to mappings with identical input and output dimensionality.

\subsection{Training}

The model is trained in a self-supervised fashion using input-output pairs of normal samples and their corrupted versions.
Let $\bfx \in [0,1]^{h \times w \times 3}$ denote a normal (anomaly-free) image.
A corrupted version $\hat{\bfx}$ is obtained by altering randomly selected regions,
specified by a real-valued mask $\mathbf{M} \in [0,1]^{h \times w \times 3}$.
Its complement is defined as $\bar{\mathbf{M}} := \mathbf{1} - \mathbf{M}$, where $\mathbf{1}$
denotes the tensor of ones.

We use a loss formulation established in prior work~\cite{math12243988}
and optimize it with respect to the model parameters $\boldsymbol{\theta}$:
\begin{equation}
\label{eq_21041620}
\mathcal{L}(\hat{\bfx}, \bfx, \mathbf{M}; \boldsymbol{\theta}) := \frac{1-\lambda}{\|\bar{\mathbf{M}}\|_1}
\|\bar{\mathbf{M}} \odot (f_{\boldsymbol{\theta}}(\hat{\bfx}) - \bfx)\|_2^2
+
\frac{\lambda}{\|\mathbf{M}\|_1}
\|\mathbf{M} \odot (f_{\boldsymbol{\theta}}(\hat{\bfx}) - \bfx)\|_2^2,
\end{equation}
where $\odot$ denotes elementwise tensor multiplication, $\|\cdot\|_p$ the $\ell^p$-norm,
and $\lambda \in [0,1]$ balances the contribution of corrupted versus uncorrupted regions.

There are multiple ways of choosing a corruption pattern for generating the training data.
An efficient and powerful augmentation technique proposed in \cite{math12243988}
is to use an additional dataset (e.g., DTD \cite{cimpoi14describing}) of background images $\mathcal{B}$
offering a variation of structural patterns not occurring in the normal data.
Given a normal image $\bfx \in \mathcal{M}$, a corruption pattern $\bfy \in \mathcal{B}$, and
a randomly selected smooth mask $\mathbf{M}$, a corrupted version is created 
according to $\hat{\bfx} = \mathbf{M} \odot \bfy + \bar{\mathbf{M}} \odot \bfx$.


\subsection{Test-Time Detection}

After training, anomaly localization is performed by comparing an input $\hat{\bfx}$ with its reconstruction $f_{\boldsymbol{\theta}}(\hat{\bfx})$. To this end, we define a pixel-wise discrepancy function
\[
\Delta :
[0,1]^{h \times w \times 3}
\times
[0,1]^{h \times w \times 3}
\rightarrow
[0,1]^{h \times w}.
\]
Different choices for $\Delta$ exist in the literature, including Mean Squared Error (MSE), Structural Similarity Index Measure (SSIM)~\cite{BergmannLFSS19}, and Gradient Magnitude Similarity (GMS)~\cite{XueZMB14}, each capturing distinct aspects of reconstruction quality.
A binary segmentation mask for anomalous regions can be extracted by thresholding the pixel-wise map of anomaly scores.
Independently of the specific choice of $\Delta$, applying spatial smoothing prior to thresholding improves stability.

For image-level detection, a global anomaly score is obtained by aggregating the values of the anomaly map, typically via summation. Alternative strategies, such as taking the maximum or restricting the aggregation to the largest responses, can reduce sensitivity to the spatial extent of anomalies. The overall inference pipeline is illustrated in Figure~\ref{fig_bauer3}.
\begin{figure}[t]
\centering
\includegraphics[scale = 0.78]{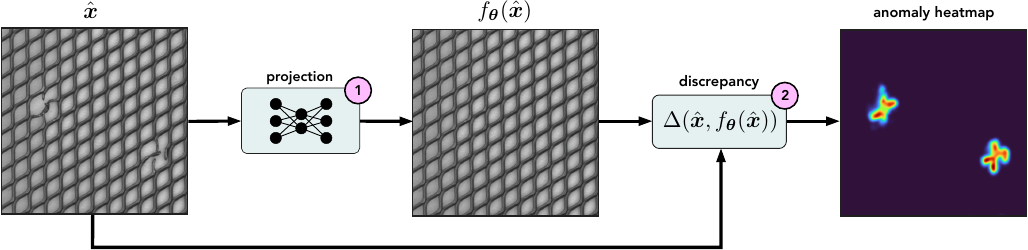}
\caption{
Illustration of our anomaly detection process.
Given an input $\hat{\bfx}$, a trained model first produces an output $f_{\bftheta}(\hat{\bfx})$,
which preserves normal regions and replaces irregularities with locally consistent patterns.
Second, we compute a pixel-wise discrepancy map
$\Delta(\hat{\bfx}, f_{\bftheta}(\hat{\bfx})) \in [0,1]^{h \times w}$
between the input and the output to localize anomalous regions.
}
\label{fig_bauer3}
\end{figure}

In essence, this conceptually simple and principled framework enables accurate and efficient structural anomaly detection, as demonstrated in the experimental section.

\section{Formal Analysis: Inductive Biases, Approximation and Generalization}
\label{sec:section4}

This section formalizes the projection-based perspective as a principled framework for structural anomaly detection (SAD).
We first show that it introduces the missing inductive bias required for manifold-supported data, providing a foundation for both detection and localization.
We then identify the conservative projection as an optimal solution to SAD and discuss how it is approximated through the training framework presented in Section~\ref{sec:section3}.
Building on this view, we provide a geometric interpretation of model generalization through contraction toward the data manifold and discuss future directions for algorithmic development, including iterative refinement through fixed-point dynamics and geometry-aware regularization.
\subsection{Geometric Inductive Bias via a Projection Operator}
\label{subsec:bias}
We now formalize the inductive biases required for detecting and localizing structural deviations.

\begin{inductivebias}[Fixed-Point Condition]
\label{def_22041512}
Let $\mathcal{M} \subset \mathbb{R}^d$ denote the set of normal data points.
A reconstruction mapping $f : \mathbb{R}^d \to \mathbb{R}^d$ must satisfy the fixed-point condition with respect to $\mathcal{M}$:
\begin{equation}
\label{E_22041718}
f(\bfx) = \bfx \quad \Longleftrightarrow \quad \bfx \in \mathcal{M}.
\end{equation}
\end{inductivebias}

This property ensures identity reconstruction exclusively on the data manifold
while excluding trivial identity mappings in the ambient space, a behavior often violated by regularized autoencoders.
It therefore constitutes a minimal requirement for anomaly detection:
any mapping satisfying this condition is sufficient for the detection task,
as off-manifold inputs necessarily violate the fixed-point condition.
In particular, any projection operator $f$ with $\operatorname{Im}(f) = \mathcal{M}$
satisfies this property.

Vanilla and many regularized autoencoders are often interpreted as learning the structure
of the underlying data manifold \cite{abs-1712-07788, pmlr-v130-connor21a}.
However, this interpretation is imprecise, as such models focus on reconstructing normal samples,
leaving their behavior in the ambient space largely unconstrained.
Moreover, commonly proposed regularization strategies, including reduced bottleneck capacity
or Jacobian penalties \cite{AlainB14}, are not sufficient to restrict near-identity reconstruction to the data manifold,
a phenomenon frequently observed in the visual domain \cite{BhattadRF18, abs-2506-10233}.
As a result, these models often provide only a weak separation signal between normal and anomalous samples.

In contrast, assuming the model accurately realizes the behavior prescribed by the inductive bias in $(\ref{E_22041718})$,
the residual mapping $\Phi(\bfx) := f(\bfx) - \bfx$ fully captures the normal data manifold as
\begin{equation}
\label{E_21041544}
\mathcal{M} = \Phi^{-1}(\mathbf{0}).
\end{equation}
If $\Phi$ is smooth and its Jacobian $D\Phi(\bfx)$ has constant rank $d-k$ in a
neighborhood of $\mathcal{M}$, then by the implicit function theorem, the level set
$\Phi^{-1}(\mathbf{0})$ locally defines a $k$-dimensional embedded submanifold of $\mathbb{R}^d$.
In this sense, the reconstruction mapping $f$ admits a natural interpretation as a projection operator
that implicitly captures the geometry of the normal manifold as its fixed-point set,
thereby separating it from the ambient space and yielding a principled criterion for
distinguishing between normal and anomalous samples.

Notably, this perspective directly aligns with the core objective of anomaly detection,
namely learning a notion of normality.
In contrast to common approaches that only approximate this objective through surrogate losses,
the proposed formulation explicitly defines normality via Eq.~(\ref{E_21041544}),
while anomalous samples are implicitly characterized by their deviation from it.
Indeed, this focus on normality is directly reflected in the projection behavior of a corresponding model,
where large regions of the ambient space map to similar points on the manifold,
rendering variation in off-manifold directions largely irrelevant.
Importantly, this significantly reduces the effective complexity of the learning problem, which is governed
by the intrinsic dimensionality of the normal manifold rather than that of the ambient space,
thereby mitigating the curse of dimensionality and improving generalization.

Beyond separating normal and anomalous samples, domains such as images (or time series) additionally require localization of anomalous regions at the pixel level.
For this purpose, we leverage the residual mapping $\Phi$, which provides a spatial discrepancy map enabling localization of anomalous regions, as introduced in Section~\ref{sec:section3}.

Given an anomalous sample identified with a vector $\hat{\bfx} \in \mathbb{R}^d$,
we define the anomalous region as a minimal set of pixels $S \subset \{1, ..., d\}$
that must be corrected to produce a normal sample $\bfx \in \mathcal{M}$.
That is, $S$ is uniquely defined and $\hat{\bfx}_{\bar{S}} = \bfx_{\bar{S}}$,
where $\bar{S} := \{1, ..., d\} \setminus S$ denotes the complement of $S$,
and $\hat{\bfx}_S$, $\bfx_S$ denote vector slices restricted to indices in $S$.
It is worth noting that the uniqueness assumption does not generally hold,
as discussed in \cite{math12243988}.
However, the resulting ambiguity arising from the finite resolution of digitized images is typically limited in extent
and can often be neglected in practice.
Formally, this ambiguity vanishes in the limit of infinite resolution.

Having established the formal definition of a \emph{structural anomaly} as an off-manifold sample $\hat{\bfx} \not\in \mathcal{M}$,
with uniquely localized anomalous part $\hat{\bfx}_S$, we now can derive an optimal projection operator
for the task of SAD,
given as 
\begin{equation}
\label{E_24041100}
\Pi_{\mathrm{con}}(\bfx) := \arg\min_{\bfy \in \mathcal{M}} \|\bfx - \bfy\|_0,
\end{equation}
which we refer to as a conservative projection.
Here, $\|\cdot\|_0$ denotes the so-called $\ell_0$ pseudo-norm,
defined as the number of nonzero components of a vector $\|\bfx\|_0 = \#\{i : x_i \neq 0\}$.
In this context, minimizing $\|\bfx-\bfy\|_0$ enforces a conservative correction
by minimizing the support of the modification, i.e., the number of altered components
required to project $\bfx$ onto $\mathcal{M}$.

Unlike orthogonal projection, which minimizes geometric distance,
the conservative projection minimizes correction support.
Since the objective $\|\bfx-\bfy\|_0$ depends on support cardinality,
the resulting operator is inherently non-smooth.
Under uniqueness assumptions on the correction support,
it may nevertheless exhibit piecewise-continuous or piecewise-differentiable behavior.
Specifically, within regions where the optimal support remains unchanged, the mapping
reduces to a constrained correction problem on the manifold, while non-smoothness arises
at transitions between distinct supports.
Consequently, such behavior may be locally approximated by a differentiable autoencoder
that learns a smooth relaxation of conservative correction.

Under uniqueness of the correction support, the conservative projection $\Pi_{\mathrm{con}}$
yields an optimal solution for SAD.
First, it naturally satisfies Inductive Bias~\ref{def_22041512}.
Anomalous regions are then identified from the individual components of the discrepancy map
$\Phi$, such that a pixel $i \in \{1,\dots,d\}$ is anomalous if and only if $|\Phi_i(\hat{\bfx})| > 0$.
In practice, however, deviations from the idealized assumptions, together with reconstruction
inaccuracies and numerical noise, motivate the use of thresholds $\tau,\nu > 0$:
specifically, $\|\Phi(\hat{\bfx})\|_2 > \tau$ for anomaly detection and $|\Phi_i(\hat{\bfx})| > \nu$ for pixel-wise localization.

The conservative projection further reveals an additional inductive bias for projection operators
required for accurate localization: minimal correction support.
While the fixed-point condition ensures that anomalous samples are altered by the reconstruction mapping $f$,
thereby enabling detection, accurate localization additionally requires preservation of normal regions.

\begin{inductivebias}[Minimal Correction Support]
\label{def_22040000}
A reconstruction mapping $f : \mathbb{R}^d \rightarrow \mathbb{R}^d$ must preserve all normal components while modifying only the anomalous regions.
Formally, for any input $\hat{\bfx}$ with anomalous support
$S \subseteq \{1,\dots,d\}$,
\begin{equation}
\label{E_23040853}
f_i(\hat{\bfx}) \neq \hat{x}_i
\quad \Longleftrightarrow \quad
i \in S.
\end{equation}
\end{inductivebias}

It is worth noting that the requirement in (\ref{E_23040853}) does not strictly imply $f(\hat{\bfx}) \in \mathcal{M}$.
While preservation of normal regions is necessary, alteration of the corrupted region is both necessary and sufficient for accurate localization.
In particular, Inductive Bias~\ref{def_22040000} implies Inductive Bias~\ref{def_22041512}.
In our framework, these inductive biases are approximated by minimizing
reconstruction loss on partially occluded images with varying opacity.
Specifically, the objective in (\ref{eq_21041620}) directly promotes preservation of normal regions
while replacing corrupted regions with manifold-consistent patterns.

\subsection{On the Approximation of the Conservative Projection Operator}
\label{subsec:sec2}
The following analysis clarifies how our training
framework encodes the inductive biases discussed in the previous section
by approximating the conservative projection operator.

For this purpose, we define the set-valued mappings
$N \colon \mathbb{R}^d \rightarrow 2^{\{1,\dots,d\}}$
and $A \colon \mathbb{R}^d \rightarrow 2^{\{1,\dots,d\}}$,
which extract the indices corresponding to the normal and anomalous regions of an input, respectively.
As discussed previously, we assume that these mappings are well-defined and that
the assignment to normal versus anomalous regions is unique.
In particular, $N(\hat{\bfx}) = \{1, ..., d\} \setminus A(\hat{\bfx})$ for all $\hat{\bfx} \in \mathbb{R}^d$.
Furthermore, we assume that a regular conditional distribution
$p(\bfx_A \mid \bfx_N)$ exists, that
$\mathbb{E}[\|\bfx\|^2] < \infty$, and that the distribution over corrupted samples
$q(\hat{\bfx})$ is such that, for $q$-a.e.\ corrupted input
$\hat{\bfx}$, the conditioning values
$\hat{\bfx}_{N(\hat{\bfx})}$ lie in the support of the marginal distribution
of $\bfx_{N(\hat{\bfx})}$ under $p$.
Finally, all conditional expectations below are also conditioned on the mechanism by which the corruption process
generates patterns on $A(\hat{\bfx})$.
We omit this dependence from the notation for simplicity.

\begin{theorem}
\label{T_23040830}
Let $p(\bfx)$ denote the distribution of normal data and let
$q(\hat{\bfx})$ denote a distribution of corrupted samples,
with $k \in \mathbb{N}$.
Consider the following unconstrained optimization problem:
\begin{equation}
\label{op_24041055}
\begin{aligned}
& \underset{\bftheta \in \mathbb{R}^k}{\text{minimize}}
& & 
\mathbb{E}_{q(\hat{\bfx})}\!\left[
\mathbb{E}_{p(\bfx)}\!\left[
\|f_{\bftheta}(\hat{\bfx})-\bfx\|^2
\;\middle|\;
\bfx_{N(\hat{\bfx})}=\hat{\bfx}_{N(\hat{\bfx})}
\right]
\right].
\end{aligned}
\end{equation}
Provided sufficient capacity of the hypothesis class $\{f_{\bftheta}\}$,
every minimizer $f^* := f_{\bftheta^*}$ satisfies, for $q$-a.e.\ $\hat{\bfx}$,
\begin{equation}
\label{E_23040803}
f^*_{N(\hat{\bfx})}(\hat{\bfx})
=
\hat{\bfx}_{N(\hat{\bfx})}
\qquad\text{and}\qquad
f^*_{A(\hat{\bfx})}(\hat{\bfx})
=
\mathbb{E}_{p(\bfx)}\!\left[
\bfx_{A(\hat{\bfx})}
\;\middle|\;
\bfx_{N(\hat{\bfx})}
=
\hat{\bfx}_{N(\hat{\bfx})}
\right].
\end{equation}
\end{theorem}

The above characterization describes the Bayes-optimal correction operator $f^*$.
For a parameterized hypothesis class $\{f_{\bftheta}\}$,
the conclusion holds whenever this class is sufficiently expressive to realize the mapping in (\ref{E_23040803}).
The training objective in (\ref{eq_21041620}) can be interpreted as an approximation
of the optimization problem in (\ref{op_24041055}), where the weighting parameter $\lambda$
controls the relative emphasis placed on the two defining properties of an optimal correction operator in (\ref{E_23040803}).
As a result, a trained model tends to preserve normal regions while replacing corrupted regions
with locally consistent patterns, thereby approximating the behavior of the conservative projection operator defined in (\ref{E_24041100}).

Theorem \ref{T_23040830} highlights an important side effect of reconstruction-based methods.
When correction of an anomalous region is ambiguous, the reconstruction loss drives the model
to produce an average over all feasible restorations under the training distribution.
Provided sufficient training data and model capacity, the approximation quality to the conservative projection
is therefore primarily limited by reconstruction ambiguity within corrupted regions.
Fortunately, in practice, this averaging effect is typically small due to sparse sampling of the training data.
Moreover, for anomaly detection, perfect reconstruction is not strictly necessary,
since deviation from the anomalous pattern is sufficient, as discussed in Section \ref{subsec:bias}.
We provide a formal proof of Theorem \ref{T_23040830} in Supplementary Section \ref{subsec:proof_theorem}.

Nevertheless, higher reconstruction fidelity in corrupted regions may be beneficial in practice,
as it can sharpen the localization signal and improve visual interpretability.
One potential solution is to replace the autoencoder with a diffusion model conditioned on the input image.
This, however, introduces substantially higher computational cost.
Instead, we propose an iterative approach that repeatedly applies the autoencoder to the input, eventually converging to a fixed point on the manifold.
Convergence is guaranteed by the contractive behavior of the model toward the data manifold, according to
\begin{equation}
\operatorname{dist}\bigl(f(\hat{\bfx}), \mathcal{M}\bigr)
\leq
\rho \cdot \operatorname{dist}\bigl(\hat{\bfx}, \mathcal{M}\bigr)
\end{equation}
for $0 < \rho < 1$,
where the contraction emerges from the specific form of our corruption process based on partial occlusions with varying opacity levels.
We provide a formal theorem and additional details in Supplementary Section~\ref{subsec:Fixed-Point}.
Figure~\ref{fig_conv} illustrates several examples of convergence toward a fixed point under iterative model application, demonstrating progressively improved reconstruction fidelity.
\begin{figure}[t]
\centering
\includegraphics[scale = 0.64]{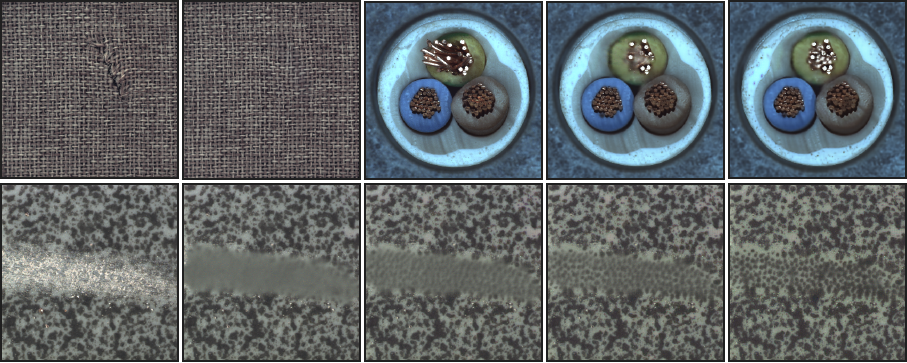}
\caption{
Illustration of improved reconstruction fidelity and varying convergence speeds to a fixed-point
under iterative application of our model, separately trained on different categories of the MVTec AD dataset.
Each example shows (from left to right) the input together with reconstruction results obtained through repeated model application,
converging toward a fixed point on the manifold.
}
\label{fig_conv}
\end{figure}

\subsection{Generalization through Contraction toward the Data Manifold}

The strong generalization ability of correction-based autoencoders can be understood
as a consequence of an emergent contraction behavior toward the data manifold.
Rather than learning separate responses to individual corruption patterns,
the model acquires a correction behavior that is largely independent of the specific perturbation
and instead restores normality by replacing irregularities with manifold-consistent content.
In particular, this contraction provides a unified explanation for why the model (surprisingly well) generalizes from artificial corruptions
introduced during training to naturally occurring anomalies present in real images,
unlike supervised segmentation approaches, where the learned behavior is primarily driven by annotation-specific supervision.

More precisely, two complementary directions of generalization must be distinguished.
First, \emph{generalization along the manifold} refers to reconstructing unseen but valid samples.
This behavior is well understood and can largely be attributed to the inductive bias of convolutional architectures,
which are naturally aligned with the spatial structure of image data.
Second, and more importantly, \emph{generalization toward the manifold}
refers to correcting arbitrary off-manifold perturbations.
This second direction cannot be explained purely through interpolation between training samples,
but instead requires understanding how a global correction behavior emerges from local supervision.

We argue that this contraction behavior arises through the interaction of three mutually reinforcing components.
First, natural images exhibit strong \emph{index-induced regularity},
meaning that valid image configurations are highly constrained by local context.
Neighboring pixels contain predictive information about one another,
making the content of corrupted or missing regions largely inferable from their surroundings.
Second, convolutional architectures are inherently aligned with this regularity, as their learned features are local in nature and shared across spatial locations.
Third, the training objective in (\ref{eq_21041620}) promotes the inductive biases established previously
by preserving normal regions while restoring corrupted regions through locally consistent content.
Collectively, these three components induce a specific contraction behavior toward the manifold,
biased toward preserving normal content while replacing irregularities with locally consistent structure.
As a consequence, this learned contraction transfers across samples and perturbations,
allowing the model to generalize beyond the specific corruptions encountered during training.

This contraction perspective also explains our empirical observation that overfitting, even after prolonged training over many epochs, is comparatively difficult in this framework.
While convolutional architectures and index-induced regularity provide the conditions under which contraction can emerge, they are not sufficient on their own, as segmentation models operate under similar architectural and data constraints.
The decisive factor is therefore the training objective, as it selectively promotes the missing inductive biases introduced previously and thereby determines which solutions become preferential during optimization.
By encouraging preservation of normal content together with restoration toward manifold-consistent structure, the objective favors shared correction behavior over sample-specific memorization.
In contrast, segmentation objectives optimize agreement with annotation masks, which may contain ambiguity, noise, or policy-dependent biases.
As a result, optimization can favor shortcut correlations tied to annotation statistics rather than to the intrinsic structure of the data, making overfitting substantially more likely.
This also aligns with our empirical observation that segmentation models may initially generalize well but increasingly overfit to annotation-specific corruption patterns as training progresses.

\subsection{Outlook and Future Investigations}

The primary goal of this paper is to motivate a new geometric perspective on structural anomaly detection.
We argued that learning a non-linear projection operator onto the manifold of normal data
provides a principled solution to the task by directly modeling normality through projection behavior.
Building on this conceptual shift, the proposed perspective opens several promising directions for future research,
supported by tools from geometric analysis.

First, as discussed previously, the projection-based perspective naturally motivates
iterative application of the model in the spirit of fixed-point theory.
In particular, this view explains the shortcomings of training with reconstruction losses
and provides a way of iteratively refining reconstruction fidelity in corrupted regions.
We provide more details in Supplementary Section~\ref{subsec:Fixed-Point}

Second, it suggests new forms of geometry-motivated regularization.
Future methods may explicitly enforce geometric properties by shaping the form and stability of the learned projection mapping.
Examples include Jacobian-based regularization to promote contraction in off-manifold directions, as well as constraints encouraging projection-specific properties such as identity on the manifold and idempotency. See Supplementary Section~\ref{subsec:jacobian} for more details.

Third, it provides a new interpretation of model generalization as a consequence of how well the learned mapping respects the geometry of the data.
In our approach, this geometry is implicitly encoded through contraction toward the data manifold.
In particular, this reduces the effective complexity of the learning problem to the intrinsic dimensionality of the underlying manifold.
We believe this geometric perspective on generalization extends beyond anomaly detection and may provide a useful framework for understanding representation learning more broadly.

\section{Experiments}
\label{sec:section5}
In this section we evaluate the training framework introduced in Section~\ref{sec:section3}.
Based on our geometric motivation we refer to the resulting model as a projecting autoencoder (PAE).
Our goal is not to establish marginal improvements on saturated benchmarks, but to empirically support the central hypothesis of this work: reconstruction-based methods that more closely approximate the inductive biases identified in Section~\ref{sec:section4} exhibit stronger and more robust anomaly detection performance.

\textbf{Datasets.}
We evaluate the performance of our method on two established benchmarks for industrial anomaly detection: MVTec AD~\cite{BergmannBFSS21} and VisA~\cite{ZouJPZD22}.
Together, these datasets cover a broad range of industrial inspection scenarios, comprising 16,175 images across 27 categories.

\textbf{Training setup.} We apply category-specific augmentations such as random rotations and flips, while reserving 5\% of the training data for validation.
A separate model is trained per category using a modified U-Net~\cite{RonnebergerFB15} with dilated bottleneck convolutions and optimized with Adam~\cite{KingmaB14} at a learning rate of $10^{-4}$.
We set $\lambda = 0.5$ in Equation~(\ref{eq_21041620}) and evaluate using the SSIM metric.
Texture categories are resized to $512 \times 512$, whereas object categories are downscaled to $256 \times 256$.
Accordingly, object categories use a larger dilation kernel ($5$ instead of $3$).
Several VisA categories contain substantial background clutter that is not annotated as anomalous.
Where necessary, we apply foreground segmentation as a preprocessing step using a U-Net trained independently of the anomaly detection task.
The resulting masks suppress irrelevant background regions prior to training, mitigating annotation noise that disproportionately affects projection-based methods.

\textbf{Results and Conceptual Comparison.}
We report the Area Under the Receiver Operating Characteristic curve (AUROC) at the image level (I-AUROC) for anomaly detection and at the pixel level (P-AUROC) for anomaly segmentation
in Table~\ref{table_results}.
In addition, we report the Per-Region Overlap (P-PRO) and Average Precision (P-AP), both computed at the pixel level.
The best and second-best results are highlighted in red and blue, respectively.
For recent competing methods, we report scores from previously published results.
Simple classical and reconstruction-based baselines are
re-evaluated in our pipeline where necessary.
Consequently, individual methods may have been evaluated under slightly different experimental conditions.
However, these differences do not affect the main observations summarized below.

Table~\ref{table_results} supports three observations consistent with our geometric interpretation.
First, boundary-based approaches~\cite{reiss2021panda} built on pretrained ResNet backbones, including OC-SVM \cite{ScholkopfPSSW01} and DeepSVDD (DSVDD) \cite{RuffGDSVBMK18}, are substantially outperformed by PAE.
Second, among reconstruction-based methods, approaches that more closely approximate projection behavior, including RIAD~\cite{ZavrtanikKS21}, CutPaste~\cite{LiSYP21}, and DRAEM~\cite{ZavrtanikKS212}, achieve markedly stronger performance than standard autoencoder AE$_{\text{SSIM}}$~\cite{BergmannLFSS19}.
However, these methods either rely on simplistic perturbations or deviate from a reconstruction objective aligned with the projection goal, resulting in a weaker approximation of the desired inductive biases.
Third, PAE remains competitive with engineered hybrid approaches, including SimpleNet~\cite{LiuZXW23}, PatchCore~\cite{RothPZSBG22}, RealNet~\cite{10658311}, PBAS~\cite{10716437}, PMSR~\cite{LI2026113331}, and GLASS~\cite{ChenLLZ24}, which lack a clear geometric interpretation.
While PAE performs on par with the strongest methods in terms of AUROC, it achieves higher AP scores.
We note that the PRO metric is highly recall-oriented and therefore less trustworthy.
This behavior is consistent with our interpretation of PAE as an approximation of a conservative projection operator that minimizes correction support, producing sharp and robust segmentation masks.
Figure~\ref{fig_examples} in the supplementary material illustrates the segmentation quality of anomalous regions achieved by PAE.

\begin{table*}[t]
\caption{Experimental results for image-level anomaly detection and pixel-level segmentation.}
\label{table_results}
\centering
\setlength{\tabcolsep}{4pt}
\renewcommand{\arraystretch}{1.08}

\scalebox{0.625}{
\begin{tabular}{ll|ccc|ccc|cccccc|>{\columncolor{gray!12}}c}
\toprule
Dataset  & Metric & AE$_{\text{SSIM}}$ & OC-SVM & DSVDD & RIAD & CutPaste & DRAEM & SimpleNet & PatchCore & RealNet & PBAS & PMSR & GLASS & \textbf{PAE} \\

\midrule
MVTec AD & I-AUROC  & 63.0 & 70.8 & 77.9 & 91.7 & 95.2 & 98.0 & 99.6 & 99.2 & 99.7 & \textcolor{blue}{99.8} & \textcolor{red}{99.9}  & \textcolor{red}{99.9} & \textcolor{red}{99.9}\\
                  & P-AUROC & 87.0 & --     &  --     & 94.2 & 96.0 & 97.3 & 98.1 & 98.4 & \textcolor{blue}{99.0} & 98.6 & 98.8  & 98.9 & \textcolor{red}{99.4}\\
                  & P-PRO      & 69.4 & --     &  --     & 89.9 & 90.1  & 91.3 & 91.1 & 92.4 & 93.5 & \textcolor{red}{97.3} & \textcolor{blue}{96.6} & 96.3 & 96.5\\
                  & P-AP         & 23.0 & --     & --      & 53.2 & 56.0 & 68.4 & 58.1 & 57.1 & \textcolor{blue}{69.4} & --     & --      & -- & \textcolor{red}{78.6}\\

\midrule
VisA & I-AUROC  & 54.9 & 59.1 & 62.2 & 81.1 & 85.1 & 88.7 & 97.1 & 94.7 & 97.8 & 97.7 & \textcolor{red}{98.5}  & \textcolor{blue}{98.2} & \textcolor{red}{98.5}\\
        & P-AUROC & 77.4 & -- & -- & 88.7 & 90.6 & 93.5 & 98.2 & 98.5 & 98.8 & \textcolor{blue}{98.8} & 98.0  & 98.6 & \textcolor{red}{98.9}\\
        & P-PRO      & 61.1 & -- & -- & 69.0 & 70.3 & 72.4  & 90.7 & 91.8 & --     & \textcolor{red}{97.1} & \textcolor{blue}{92.3}  & 90.8 & 82.2\\
        & P-AP         & 17.4 & -- & -- & 24.8 & 24.2 & 26.6 & 37.3      & --     & --     & \textcolor{blue}{47.6} & --      & --     & \textcolor{red}{53.8}\\

\midrule
Mean & I-AUROC   & 59.0 & 65.0 & 70.1 & 86.4 & 90.2 & 93.4 & 98.4 & 97.0 & 98.8 & 98.8 & \textbf{\textcolor{red}{99.2}} & \textcolor{blue}{99.1} & \textbf{\textcolor{red}{99.2}}\\
          & P-AUROC & 82.2 & --  & -- & 91.5 & 93.3 & 95.4 & 98.2 & 98.5 & \textcolor{blue}{98.9} & 98.7 & 98.4 & 98.8 & \textbf{\textcolor{red}{99.2}}\\
          & P-PRO      & 65.3 & --  & -- & 79.5 & 80.2 & 81.9 & 90.9 & 92.1 & --      & \textbf{\textcolor{red}{97.2}} & \textcolor{blue}{94.5} & 93.6 & 89.4\\
          & P-AP         & 20.2 & --  & -- & 39.0 & 40.1 & 47.5 & \textcolor{blue}{47.7} & --      & --      & --     & --      & -- & \textbf{\textcolor{red}{66.2}}\\

\bottomrule
\end{tabular}}
\end{table*}

\section{Conclusion}
\label{sec:section6}

We argued that classical decision-boundary approaches to AD are fundamentally misaligned
with manifold-supported data, where normal samples occupy a low-dimensional subspace embedded in the ambient space.
In high-dimensional settings, modeling normality through enclosing boundaries or density estimation
becomes increasingly ill-posed and susceptible to the curse of dimensionality.
Instead, we introduced a geometric perspective based on learning a non-linear projection operator onto the manifold of normal data.
We identified the inductive biases required for SAD, namely the fixed-point condition and minimal correction support,
and introduced the conservative projection operator as the corresponding optimal solution.

Furthermore, the projection perspective explains the success of correction-based reconstruction methods
through contraction toward the manifold.
This contraction explains why models generalize from artificial corruptions introduced during training
to naturally occurring anomalies in real images.
Our training objective promotes preservation of normal content together with restoration toward locally consistent patterns.
As a result, the learned mapping becomes increasingly determined by data geometry rather than by individual corruption patterns.
Beyond enabling generalization, contraction toward the manifold effectively reduces the complexity of the learning problem
to the intrinsic dimensionality of the manifold, thereby mitigating the curse of dimensionality.
While the proposed formulation provides a principled framework for SAD,
it relies on deviations in local predictive structure. In contrast, logical anomalies may fully preserve
the appearance statistics of natural images, producing no projection residual and therefore cannot be fully addressed by our framework.

The geometric perspective also opens several directions for future research.
Projection operators naturally connect AD to fixed-point theory,
enabling iterative refinement schemes and a dynamical interpretation of convergence toward the manifold.
Moreover, the geometric formulation motivates new regularization strategies, including Jacobian-based constraints
that encourage contraction in off-manifold directions while preserving variation along the manifold,
offering a foundation for future theoretical and algorithmic developments.

%

\bibliography{references}

@article{ScholkopfPSSW01,
  author    = {Bernhard Sch{\"{o}}lkopf and
               John C. Platt and
               John Shawe{-}Taylor and
               Alexander J. Smola and
               Robert C. Williamson},
  title     = {Estimating the Support of a High-Dimensional Distribution},
  journal   = {Neural Comput.},
  volume    = {13},
  number    = {7},
  pages     = {1443--1471},
  year      = {2001}
}

@article{TaxD04,
  author    = {David M. J. Tax and
               Robert P. W. Duin},
  title     = {Support Vector Data Description},
  journal   = {Mach. Learn.},
  volume    = {54},
  number    = {1},
  pages     = {45--66},
  year      = {2004}
}

@inproceedings{RuffGDSVBMK18,
  author    = {Lukas Ruff and
               Nico G{\"o}rnitz and
               Lucas Deecke and
               Shoaib Ahmed Siddiqui and
               Robert A. Vandermeulen and
               Alexander Binder and
               Emmanuel M{\"u}ller and
               Marius Kloft},
  editor    = {Jennifer G. Dy and
               Andreas Krause},
  title     = {Deep One-Class Classification},
  booktitle = {Proc. 35th Int. Conf. Mach. Learn.},
  series    = {Proc. Mach. Learn. Res.},
  volume    = {80},
  pages     = {4390--4399},
  year      = {2018}
}

@article{ZavrtanikKS21,
  author    = {Vitjan Zavrtanik and
               Matej Kristan and
               Danijel Skocaj},
  title     = {Reconstruction by inpainting for visual anomaly detection},
  journal   = {Pattern Recognit.},
  volume    = {112},
  pages     = {107706},
  year      = {2021}
}

@inproceedings{LiSYP21,
  author    = {Chun-Liang Li and
               Kihyuk Sohn and
               Jinsung Yoon and
               Tomas Pfister},
  title     = {CutPaste: Self-Supervised Learning for Anomaly Detection and Localization},
  booktitle = {Proc. IEEE Conf. Comput. Vis. Pattern Recognit.},
  pages     = {9664--9674},
  year      = {2021}
}

@inproceedings{InTra,
  author    = {Jonathan Pirnay and
               Keng Chai},
  title     = {Inpainting Transformer for Anomaly Detection},
  booktitle = {Proc. Int. Conf. Image Anal. Process.},
  series    = {Lect. Notes Comput. Sci.},
  volume    = {13232},
  pages     = {394--406},
  year      = {2022}
}

@article{BhattadRF18,
  author    = {Anand Bhattad and
               Jason Rock and
               David A. Forsyth},
  title     = {Detecting Anomalous Faces with ``No Peeking'' Autoencoders},
  journal   = {CoRR},
  volume    = {abs/1802.05798},
  year      = {2018},
  eprint    = {1802.05798},
  archivePrefix = {arXiv},
  primaryClass  = {cs.CV}
}

@article{BergmannBFSS21,
  author    = {Paul Bergmann and
               Kilian Batzner and
               Michael Fauser and
               David Sattlegger and
               Carsten Steger},
  title     = {The MVTec Anomaly Detection Dataset: {A} Comprehensive Real-World
               Dataset for Unsupervised Anomaly Detection},
  journal   = {Int. J. Comput. Vis.},
  volume    = {129},
  number    = {4},
  pages     = {1038--1059},
  year      = {2021}
}

@inproceedings{BergmannLFSS19,
  author    = {Paul Bergmann and
               Sindy L{\"o}we and
               Michael Fauser and
               David Sattlegger and
               Carsten Steger},
  title     = {Improving Unsupervised Defect Segmentation by Applying Structural
               Similarity to Autoencoders},
  booktitle = {Proc. Int. Conf. Comput. Vis. Theory Appl.},
  pages     = {372--380},
  year      = {2019}
}

@inproceedings{KingmaB14,
  author    = {Diederik P. Kingma and
               Jimmy Ba},
  title     = {Adam: {A} Method for Stochastic Optimization},
  booktitle = {Proc. Int. Conf. Learn. Represent.},
  year      = {2015}
}

@inproceedings{ZavrtanikKS212,
  author    = {Vitjan Zavrtanik and
               Matej Kristan and
               Danijel Skocaj},
  title     = {DR{\AE}M -- {A} Discriminatively Trained Reconstruction Embedding for
               Surface Anomaly Detection},
  booktitle = {Proc. IEEE Int. Conf. Comput. Vis.},
  pages     = {8310--8319},
  year      = {2021}
}

@inproceedings{LiuZXW23,
  author    = {Zhikang Liu and
               Yiming Zhou and
               Yuansheng Xu and
               Zilei Wang},
  title     = {SimpleNet: {A} Simple Network for Image Anomaly Detection and Localization},
  booktitle = {Proc. IEEE Conf. Comput. Vis. Pattern Recognit.},
  pages     = {20402--20411},
  year      = {2023}
}

@inproceedings{RothPZSBG22,
  author    = {Karsten Roth and
               Latha Pemula and
               Joaquin Zepeda and
               Bernhard Sch{\"o}lkopf and
               Thomas Brox and
               Peter V. Gehler},
  title     = {Towards Total Recall in Industrial Anomaly Detection},
  booktitle = {Proc. IEEE Conf. Comput. Vis. Pattern Recognit.},
  pages     = {14298--14308},
  year      = {2022}
}

@inproceedings{ZouJPZD22,
  author    = {Yang Zou and
               Jongheon Jeong and
               Latha Pemula and
               Dongqing Zhang and
               Onkar Dabeer},
  title     = {SPot-the-Difference Self-Supervised Pre-Training for Anomaly Detection
               and Segmentation},
  booktitle = {Proc. Eur. Conf. Comput. Vis.},
  series    = {Lect. Notes Comput. Sci.},
  volume    = {13690},
  pages     = {392--408},
  year      = {2022}
}

@article{RuffKVMSKDM21,
  author    = {Lukas Ruff and
               Jacob R. Kauffmann and
               Robert A. Vandermeulen and
               Wojciech Samek and
               Marius Kloft and
               Thomas G. Dietterich and
               Klaus-Robert M{\"u}ller},
  title     = {A Unifying Review of Deep and Shallow Anomaly Detection},
  journal   = {Proc. IEEE},
  volume    = {109},
  number    = {5},
  pages     = {756--795},
  year      = {2021}
}

@article{OptMPMP,
  author    = {Alexander Bauer and
               Shinichi Nakajima and
               Nico G{\"o}rnitz and
               Klaus-Robert M{\"u}ller},
  title     = {Optimizing for Measure of Performance in Max-Margin Parsing},
  journal   = {IEEE Trans. Neural Netw. Learn. Syst.},
  pages     = {1--5},
  year      = {2019}
}

@article{BauerSM17,
  author    = {Alexander Bauer and
               Shinichi Nakajima and
               Klaus-Robert M{\"u}ller},
  title     = {Efficient Exact Inference with Loss Augmented Objective in Structured Learning},
  journal   = {IEEE Trans. Neural Netw. Learn. Syst.},
  volume    = {28},
  number    = {11},
  pages     = {2566--2579},
  year      = {2017}
}

@inproceedings{Bauer2019,
  author    = {Alexander Bauer and
               Shinichi Nakajima and
               Nico G{\"o}rnitz and
               Klaus-Robert M{\"u}ller},
  title     = {Partial Optimality of Dual Decomposition for {MAP} Inference in Pairwise {MRF}s},
  booktitle = {Proc. Int. Conf. Artif. Intell. Stat.},
  series    = {Proc. Mach. Learn. Res.},
  volume    = {89},
  pages     = {1696--1703},
  year      = {2019}
}

@article{math11122628,
  author    = {Alexander Bauer and
               Shinichi Nakajima and
               Klaus-Robert M{\"u}ller},
  title     = {Polynomial-Time Constrained Message Passing for Exact {MAP} Inference
               on Discrete Models with Global Dependencies},
  journal   = {Mathematics},
  volume    = {11},
  number    = {12},
  pages     = {2628},
  year      = {2023}
}

@article{math12243988,
  author    = {Alexander Bauer and
               Shinichi Nakajima and
               Klaus-Robert M{\"u}ller},
  title     = {Self-Supervised Autoencoders for Visual Anomaly Detection},
  journal   = {Mathematics},
  volume    = {12},
  number    = {24},
  pages     = {3988},
  year      = {2024}
}

@article{Parzen62,
  author  = {E. Parzen},
  title   = {On Estimation of a Probability Density Function and Mode},
  journal = {Ann. Math. Stat.},
  volume  = {33},
  number  = {3},
  pages   = {1065--1076},
  year    = {1962}
}

@article{KimKYC23,
  author  = {Minkyung Kim and
             Junsik Kim and
             Jongmin Yu and
             Jun Kyun Choi},
  title   = {Active Anomaly Detection Based on Deep One-Class Classification},
  journal = {Pattern Recognit. Lett.},
  volume  = {167},
  pages   = {18--24},
  year    = {2023}
}

@inproceedings{TackMJS20,
  author    = {Jihoon Tack and
               Sangwoo Mo and
               Jongheon Jeong and
               Jinwoo Shin},
  title     = {{CSI}: Novelty Detection via Contrastive Learning on Distributionally
               Shifted Instances},
  booktitle = {Adv. Neural Inf. Process. Syst.},
  year      = {2020}
}

@article{AlainB14,
  author  = {Guillaume Alain and
             Yoshua Bengio},
  title   = {What Regularized Auto-Encoders Learn from the Data-Generating Distribution},
  journal = {J. Mach. Learn. Res.},
  volume  = {15},
  number  = {110},
  pages   = {3563--3593},
  year    = {2014}
}

@inproceedings{RonnebergerFB15,
  author    = {Olaf Ronneberger and
               Philipp Fischer and
               Thomas Brox},
  title     = {U-Net: Convolutional Networks for Biomedical Image Segmentation},
  booktitle = {Proc. Med. Image Comput. Comput.-Assist. Interv.},
  series    = {Lect. Notes Comput. Sci.},
  volume    = {9351},
  pages     = {234--241},
  year      = {2015}
}

@inproceedings{cimpoi14describing,
  author    = {M. Cimpoi and
               S. Maji and
               I. Kokkinos and
               S. Mohamed and
               A. Vedaldi},
  title     = {Describing Textures in the Wild},
  booktitle = {Proc. IEEE Conf. Comput. Vis. Pattern Recognit.},
  year      = {2014}
}

@article{XueZMB14,
  author  = {Wufeng Xue and
             Lei Zhang and
             Xuanqin Mou and
             Alan C. Bovik},
  title   = {Gradient Magnitude Similarity Deviation: {A} Highly Efficient Perceptual
             Image Quality Index},
  journal = {IEEE Trans. Image Process.},
  volume  = {23},
  number  = {2},
  pages   = {684--695},
  year    = {2014}
}

@inproceedings{10658311,
  author    = {Ximiao Zhang and
               Min Xu and
               Xiuzhuang Zhou},
  title     = {RealNet: {A} Feature Selection Network with Realistic Synthetic Anomaly
               for Anomaly Detection},
  booktitle = {Proc. IEEE Conf. Comput. Vis. Pattern Recognit.},
  pages     = {16699--16708},
  year      = {2024}
}

@inproceedings{ChenLLZ24,
  author    = {Qiyu Chen and
               Huiyuan Luo and
               Chengkan Lv and
               Zhengtao Zhang},
  title     = {A Unified Anomaly Synthesis Strategy with Gradient Ascent for Industrial
               Anomaly Detection and Localization},
  booktitle = {Proc. Eur. Conf. Comput. Vis.},
  series    = {Lect. Notes Comput. Sci.},
  volume    = {15125},
  pages     = {37--54},
  year      = {2024}
}

@article{LI2026113331,
  author  = {Lanxiao Li and
             Chunjuan Yan and
             Da Song and
             Binghui Wang and
             Chuanxu Wang},
  title   = {A Prototype Correction Multi-Scale Feature Reconstruction Network for
             Industrial Anomaly Detection},
  journal = {Pattern Recognit.},
  volume  = {177},
  pages   = {113331},
  year    = {2026}
}

@article{abs-1712-07788,
  author        = {Dejiao Zhang and
                   Yifan Sun and
                   Brian Eriksson and
                   Laura Balzano},
  title         = {Deep Unsupervised Clustering Using Mixture of Autoencoders},
  journal       = {CoRR},
  volume        = {abs/1712.07788},
  year          = {2017},
  eprint        = {1712.07788},
  archivePrefix = {arXiv}
}

@inproceedings{pmlr-v130-connor21a,
  author    = {Marissa Connor and
               Gregory Canal and
               Christopher Rozell},
  title     = {Variational Autoencoder with Learned Latent Structure},
  booktitle = {Proc. Int. Conf. Artif. Intell. Stat.},
  series    = {Proc. Mach. Learn. Res.},
  volume    = {130},
  pages     = {2359--2367},
  year      = {2021}
}

@article{abs-2506-10233,
  author        = {Ana Lawry Aguila and
                   Peirong Liu and
                   Oula Puonti and
                   Juan Eugenio Iglesias},
  title         = {Conditional Diffusion Models for Guided Anomaly Detection in Brain
                   Images Using Fluid-Driven Anomaly Randomization},
  journal       = {CoRR},
  volume        = {abs/2506.10233},
  year          = {2025},
  eprint        = {2506.10233},
  archivePrefix = {arXiv}
}

@inproceedings{reiss2021panda,
  author    = {Tal Reiss and
               Niv Cohen and
               Liron Bergman and
               Yedid Hoshen},
  title     = {{PANDA}: Adapting Pretrained Features for Anomaly Detection and Segmentation},
  booktitle = {Proc. IEEE Conf. Comput. Vis. Pattern Recognit.},
  pages     = {2806--2814},
  year      = {2021}
}

@article{10716437,
  author    = {Qiyu Chen and
               Huiyuan Luo and
               Han Gao and
               Chengkan Lv and
               Zhengtao Zhang},
  title     = {Progressive Boundary Guided Anomaly Synthesis for Industrial Anomaly Detection},
  journal   = {IEEE Trans. Circuits Syst. Video Technol.},
  volume    = {35},
  number    = {2},
  pages     = {1193--1208},
  year      = {2025}
}
\bibliographystyle{abbrv}


\appendix
\section{Supplementary Material}

This supplementary material provides additional theoretical analysis and proofs supporting the main paper.
Section~\ref{subsec:proof_theorem} contains a formal proof for Theorem \ref{T_23040830}.
Section~\ref{subsec:Fixed-Point} discusses theoretical guarantees for convergence to a fixed-point on the manifold
under iterative model application.
Section~\ref{subsec:jacobian} provides additional discussion of Jacobian-based regularization.
We note that the theorems and proofs presented here remain preliminary and may be refined in future revisions.

\subsection{Proof of Theorem \ref{T_23040830}}
\label{subsec:proof_theorem}

We first prove a simpler claim.
\begin{lemma}
\label{lemma_08071708}
Let $p(\bfx)$ denote the probability distribution of normal data and
$\hat{\bfx} \in \mathbb{R}^d$ a fixed anomalous sample with normal region $N(\hat{\bfx})$.
Consider the unconstrained optimization problem
\begin{equation}
\begin{aligned}
& \underset{\bftheta \in \mathbb{R}^d}{\text{minimize}}
& & \mathbb{E}_{p(\bfx)}\!\left[\|f_{\bftheta}(\hat{\bfx}) - \bfx\|^2 \,\middle|\,
\bfx_{N(\hat{\bfx})} = \hat{\bfx}_{N(\hat{\bfx})} \right]\\
\end{aligned}
\end{equation}
Provided sufficient capacity of the hypothesis class $\{f_{\bftheta}\}$,
every minimizer $f^* := f_{\bftheta^*}$ satisfies
\begin{equation}
\label{E_07061828}
f^*_{N(\hat{\bfx})}(\hat{\bfx}) = \hat{\bfx}_{N(\hat{\bfx})}
\qquad\text{and}\qquad
f^*_{A(\hat{\bfx})}(\hat{\bfx})
= \mathbb{E}_{p(\bfx)}\!\left[\bfx_{A(\hat{\bfx})} \,\middle|\,
\bfx_{N(\hat{\bfx})} = \hat{\bfx}_{N(\hat{\bfx})}\right].
\end{equation}
\end{lemma}

\begin{proof}
We write $S := A(\hat{\bfx})$ and $\bar{S} := N(\hat{\bfx})$.
Consider the following derivations:
\begin{equation*}
\small
\begin{aligned}
&\hspace*{12pt}\mathbb{E}_{p(\bfx)}\left[\|f(\hat{\bfx}) - \bfx\|^2 \hspace*{2pt} | \hspace*{2pt} \bfx_{\bar{S}} = \hat{\bfx}_{\bar{S}} \right]\\
&= \int \|f(\hat{\bfx}) - \bfx\|^2 p(\bfx | \bfx_{\bar{S}} = \hat{\bfx}_{\bar{S}}) \mathrm{d}\bfx\\
&= \int \left( \|f(\hat{\bfx})\|^2 -2f(\hat{\bfx})^{\top} \bfx + \|\bfx\|^2 \right) p(\bfx | \bfx_{\bar{S}} = \hat{\bfx}_{\bar{S}}) \mathrm{d}\bfx\\
&= \|f(\hat{\bfx})\|^2 \underbrace{\int p(\bfx | \bfx_{\bar{S}} = \hat{\bfx}_{\bar{S}}) \mathrm{d}\bfx}_{= 1} - 2 f(\hat{\bfx})^{\top}\underbrace{\int \bfx p(\bfx | \bfx_{\bar{S}} = \hat{\bfx}_{\bar{S}}) \mathrm{d}\bfx}_{\mathbb{E}[\bfx | \bfx_{\bar{S}} = \hat{\bfx}_{\bar{S}}]} + \underbrace{\int \|\bfx\|^2 p(\bfx | \bfx_{\bar{S}} = \hat{\bfx}_{\bar{S}}) \mathrm{d}\bfx}_{\mathbb{E}[\|\bfx\|^2 | \bfx_{\bar{S}} = \hat{\bfx}_{\bar{S}}]}\\
&= \|f(\hat{\bfx})\|^2 - 2 f(\hat{\bfx})^{\top} \mathbb{E}[\bfx | \bfx_{\bar{S}} = \hat{\bfx}_{\bar{S}}] + \mathbb{E}[\|\bfx\|^2 | \bfx_{\bar{S}} = \hat{\bfx}_{\bar{S}}].
\end{aligned}
\end{equation*}
Note that $\bfmu := f(\hat{\bfx})$ in the above derivation is a constant with respect to the expectation.
We now look for stationary points of the mapping $g \colon \bfmu \mapsto \|\bfmu\|^2 - 2 \bfmu^{\top} \mathbb{E}[\bfx | \bfx_{\bar{S}} = \hat{\bfx}_{\bar{S}}] + \mathbb{E}[\|\bfx\|^2 | \bfx_{\bar{S}} = \hat{\bfx}_{\bar{S}}]$
by setting its gradient to zero and solving for $\bfmu$.
It follows:
\begin{equation*}
\nabla_{\bfmu}g = 2\bfmu^{\top} - 2 \mathbb{E}[\bfx | \bfx_{\bar{S}} = \hat{\bfx}_{\bar{S}}]^{\top} = \bf0^{\top}.
\end{equation*}
Since $g(\bfmu)$ is a convex function, this implies that $f^*(\hat{\bfx}) = \mathbb{E}_{p(\bfx)}[\bfx | \bfx_{\bar{S}} = \hat{\bfx}_{\bar{S}}]$ is the unique optimal output.
Note that $f^*$ itself does not need to be unique.
In particular, any minimizer $f^*$ satisfies $f^*_{S}(\hat{\bfx}) = \mathbb{E}_{p(\bfx)}[\bfx_{S} | \bfx_{\bar{S}} = \hat{\bfx}_{\bar{S}}]$
and $f^*_{\bar{S}}(\hat{\bfx}) = \mathbb{E}_{p(\bfx)}[\bfx_{\bar{S}} | \bfx_{\bar{S}} = \hat{\bfx}_{\bar{S}}] = \hat{\bfx}_{\bar{S}}$.
\end{proof}

\textbf{We now prove Theorem} \ref{T_23040830}.
\begin{proof}
For each fixed $\hat{\bfx}$, define the conditional risk
\[
R_{\hat{\bfx}}(f)
:=
\mathbb{E}_{p(\bfx)}\!\left[
\|f(\hat{\bfx})-\bfx\|^2
\;\middle|\;
\bfx_{N(\hat{\bfx})}=\hat{\bfx}_{N(\hat{\bfx})}
\right].
\]

By Lemma~\ref{lemma_08071708}, every function $f^*$ satisfying (\ref{E_07061828}) minimizes
$R_{\hat{\bfx}}(\cdot)$ for the fixed input $\hat{\bfx}$. Therefore, for any
$f$ and every admissible $\hat{\bfx}$ we have
\[
R_{\hat{\bfx}}(f) \ge R_{\hat{\bfx}}(f^*).
\]
Taking expectation with respect to $\hat{\bfx}\sim q$ yields
\begin{equation}
\label{E_07062136}
\mathbb{E}_{q(\hat{\bfx})}[R_{\hat{\bfx}}(f)]
\;\ge\;
\mathbb{E}_{q(\hat{\bfx})}[R_{\hat{\bfx}}(f^*)].
\end{equation}

Moreover, since for every admissible $\hat{\bfx}$ the optimal output value
$f^*(\hat{\bfx})$ is unique, equality in (\ref{E_07062136}) can hold only if
the output $f(\hat{\bfx})$ equals this unique value for $q$-almost every $\hat{\bfx}$.
Hence, every minimizer $f^*$ of the
optimization problem in Theorem~\ref{T_23040830} satisfies
\[
f^*_{N(\hat{\bfx})}(\hat{\bfx})
=
\hat{\bfx}_{N(\hat{\bfx})}
\qquad\text{and}\qquad
f^*_{A(\hat{\bfx})}(\hat{\bfx})
=
\mathbb{E}_{p(\bfx)}\!\left[
\bfx_{A(\hat{\bfx})}
\;\middle|\;
\bfx_{N(\hat{\bfx})}
=
\hat{\bfx}_{N(\hat{\bfx})}
\right]
\]
for $q$-almost every $\hat{\bfx}$.

\end{proof}
As mentioned in the main body of the paper in Section \ref{subsec:sec2},
the conditional expectations are also conditioned, beyond the values on $N(\hat{\bfx})$,
on the mechanism by which the corruption process $\mathcal{C}$
generates patterns on $A(\hat{\bfx})$.
Therefore, the precise form of the solution in (\ref{E_23040803}) is
\begin{equation}
\label{E_08071526}
f^*_{A(\hat{\bfx})}(\hat{\bfx})
=
\mathbb{E}_{p(\bfx)}\!\left[
\bfx_{A(\hat{\bfx})}
\;\middle|\;
\bfx_{N(\hat{\bfx})}
=
\hat{\bfx}_{N(\hat{\bfx})}, \mathcal{C}_{A(\hat{\bfx})}(\bfx) = \hat{\bfx}_{A(\hat{\bfx})}
\right].
\end{equation}

\subsection{Fixed-Point Convergence under Contraction toward the Manifold}
\label{subsec:Fixed-Point}
The following theorem establishes convergence to a fixed point on the manifold
under repeated application of the model, assuming a contraction property toward the data manifold.

As a first simplifying modeling assumption, we assume that for any given anomalous sample
$\hat{\bfx} \in \mathbb{R}^d \setminus \mathcal{M}$, the anomalous index set
$S := A(\hat{\bfx})$ is uniquely identifiable, since ambiguous boundary cases occur with negligible probability in practice.

As an additional simplifying assumption, we assume that for a given anomalous input
$\hat{\bfx}$, the anomaly index set $S$ remains unchanged (up until convergence)
under iterative application of the model.

\begin{theorem}
\label{T_08071253}
Let $\mathcal{M} \subset \mathbb{R}^d$ be a closed manifold and 
let $f : \mathbb{R}^d \rightarrow \mathbb{R}^d$ be a non-linear operator satisfying:

\begin{enumerate}
\item[(i)] (Conservativity) 
For every input $\hat{\bfx} \in \mathbb{R}^d$ with anomaly index set 
$S$, the operator preserves the complementary 
coordinates corresponding to normal regions:
\[
f_{\bar{S}}(\hat{\bfx}) = \hat{\bfx}_{\bar{S}}.
\]

\item[(ii)] (Slice Contraction) 
There exists $\rho \in (0,1)$ such that for all $\bfu \in \mathbb{R}^d$,
\[
\operatorname{dist}\bigl(f_{S}(\bfu), \mathcal{M}_{S}(\bfu)\bigr)
\leq
\rho \cdot \operatorname{dist}\bigl(\bfu_{S}, \mathcal{M}_{S}(\bfu)\bigr),
\]
where the slice
\[
\mathcal{M}_S(\bfu)
:=
\{\bfx_S \colon \bfx \in \mathcal{M},\ \bfx_{\bar{S}} = \bfu_{\bar{S}}\}
\]
and
\[
\operatorname{dist}(z,A)
:=
\inf_{y\in A} \|z-y\|_2
\qquad
\text{for } z\in\mathbb{R}^k,\ A\subset\mathbb{R}^k,\ k=|S|.
\]

\item[(iii)] (Vanishing correction near the slice)
There exists $C>0$ such that for all $\bfu \in \mathbb{R}^d$,
\[
\|f_S(\bfu)-\bfu_S\|_2
\le
C\,\operatorname{dist}\bigl(\bfu_S,\mathcal{M}_S(\bfu)\bigr).
\]
\end{enumerate}

Then the following statements hold:

\begin{enumerate}
\item[(a)] 
A point $\bfx \in \mathbb{R}^d$ belongs to $\mathcal{M}$ if and only if 
it is a fixed point of $f$.

\item[(b)] 
For every anomalous sample $\hat{\bfx}$ with anomaly index set $S$, 
there exists $\bfx^\ast \in \mathcal{M}$ such that
$\bfx^\ast_{\bar{S}} = \hat{\bfx}_{\bar{S}}$ and
\[
\lim_{n \to \infty} f^{(n)}(\hat{\bfx}) = \bfx^\ast .
\]
\end{enumerate}
\end{theorem}

\begin{proof}
We prove (a) and (b).

\medskip
\noindent\textbf{(a)} ($\Rightarrow$)
Let $\bfx\in\mathcal{M}$. By the convention,
normal samples have empty anomaly index set, i.e.\ $S=\emptyset$ and hence $\bar S=\{1,\dots,d\}$.
Applying (i) yields
\begin{equation*}
f(\bfx) = f_{\bar S}(\bfx) = \bfx_{\bar{S}} = \bfx,
\end{equation*}
so $\bfx$ is a fixed point of $f$.

\medskip
\noindent\textbf{(a)} ($\Leftarrow$)
Let $\bfx\in\mathbb{R}^d$ be a fixed point, i.e.\ $f(\bfx)=\bfx$, and let $S$ be its anomaly index set.
If $S=\emptyset$, then $\bfx$ is normal by definition and hence $\bfx\in\mathcal{M}$.

Assume $S\neq\emptyset$. Using (ii) with $\bfu=\bfx$ and $f_S(\bfx)=\bfx_S$ gives
\[
\operatorname{dist}\bigl(\bfx_S,\mathcal{M}_S(\bfx)\bigr)
=
\operatorname{dist}\bigl(f_S(\bfx),\mathcal{M}_S(\bfx)\bigr)
\le
\rho\,\operatorname{dist}\bigl(\bfx_S,\mathcal{M}_S(\bfx)\bigr).
\]
Since $\rho\in(0,1)$, we obtain
\[
\operatorname{dist}(\bfx_S,\mathcal{M}_S(\bfx))=0.
\]

By definition of $\mathcal{M}_S(\bfx)$, there exists a sequence
$(\bfy^{(n)})\subset\mathcal{M}$ satisfying
\[
\bfy^{(n)}_{\bar S}=\bfx_{\bar S}
\qquad\text{and}\qquad
\bfy^{(n)}_S \to \bfx_S .
\]
Hence $\bfy^{(n)}\to\bfx$ in $\mathbb{R}^d$.
Moreover, the set
\[
\{\bfy\in\mathcal{M}:\ \bfy_{\bar S}=\bfx_{\bar S}\}
\]
is closed as the intersection of the closed set $\mathcal{M}$ with an affine subspace.
Since each $\bfy^{(n)}$ belongs to a closed set and $\bfy^{(n)}\to\bfx$,
we conclude that the limit $\bfx$ also belongs to this set.
Therefore $\bfx\in\mathcal{M}$.

\medskip
\noindent\textbf{(b)}
Fix an anomalous sample $\hat{\bfx}$ with anomaly index set $S$ and define the iterates
\[
\bfx^{(0)}:=\hat{\bfx},
\qquad
\bfx^{(n+1)}:=f(\bfx^{(n)}).
\]

By (i), for all $n\ge 0$,
\[
\bfx^{(n)}_{\bar S}=\hat{\bfx}_{\bar S}.
\]

Consequently, the slice is constant along the iterates:
\[
\mathcal{M}_S(\bfx^{(n)})
=
\{\bfx_S:\bfx\in\mathcal{M},\ \bfx_{\bar S}=\bfx^{(n)}_{\bar S}\}
=
\{\bfx_S:\bfx\in\mathcal{M},\ \bfx_{\bar S}=\hat{\bfx}_{\bar S}\}
=: \mathcal{A}.
\]

Define
\[
\delta_n := \operatorname{dist}(\bfx^{(n)}_S,\mathcal{A}).
\]

Applying (ii) with $\bfu=\bfx^{(n)}$ yields
\[
\delta_{n+1}
=
\operatorname{dist}\bigl(\bfx^{(n+1)}_S,\mathcal{A}\bigr)
=
\operatorname{dist}\bigl(f_S(\bfx^{(n)}),\mathcal{M}_S(\bfx^{(n)})\bigr)
\le
\rho\,\operatorname{dist}\bigl(\bfx^{(n)}_S,\mathcal{M}_S(\bfx^{(n)})\bigr)
=
\rho\,\delta_n,
\]
and hence
\[
\delta_n \le \rho^n\delta_0
\qquad
\text{for all } n.
\]

We now show that $(\bfx^{(n)}_S)$ is a Cauchy sequence. By (iii),
\[
\|\bfx^{(n+1)}_S-\bfx^{(n)}_S\|_2
=
\|f_S(\bfx^{(n)})-\bfx^{(n)}_S\|_2
\le
C\,\operatorname{dist}\bigl(\bfx^{(n)}_S,\mathcal{M}_S(\bfx^{(n)})\bigr)
=
C\,\delta_n
\le
C\,\rho^n\delta_0.
\]

Therefore, for $m > n$,
\begin{equation*}
\begin{aligned}
\|\bfx^{(m)}_S-\bfx^{(n)}_S\|_2
&\le
\sum_{t=n}^{m-1}\|\bfx^{(t+1)}_S-\bfx^{(t)}_S\|_2
\le
C\delta_0\sum_{t=n}^{m-1}\rho^t
=
C\delta_0 \left(\sum_{t=0}^{m-1}\rho^t - \sum_{t=0}^{n-1}\rho^t \right) \\
&=
C\delta_0 \left(\frac{1 - \rho^m}{1 - \rho} - \frac{1 - \rho^n}{1 - \rho} \right)
\le
\frac{C\delta_0}{1-\rho}\,\rho^n
\xrightarrow[n\to\infty]{} 0.
\end{aligned}
\end{equation*}

Thus $(\bfx^{(n)}_S)$ converges to some limit
$\bfx^\ast_S\in\mathbb{R}^{|S|}$.

Since $\operatorname{dist}(\bfx^{(n)}_S,\mathcal{A}) = \delta_n \leq \rho^n\delta_0\to 0$
and $\bfx^{(n)}_S\to\bfx^\ast_S$, it follows that $\operatorname{dist}(\bfx^\ast_S,\mathcal{A})=0$.
Because $\mathcal{A}$ is closed, we conclude that $\bfx^\ast_S\in\mathcal{A}$.

Moreover, $\bfx^{(n)}_{\bar S}=\hat{\bfx}_{\bar S}$ for all $n$,
hence $\bfx^{(n)}\to\bfx^\ast$ where
$\bfx^\ast_{\bar S}=\hat{\bfx}_{\bar S}$ and $\bfx^\ast_S\in\mathcal{A}$.
By definition of $\mathcal{A}$, this implies $\bfx^\ast\in\mathcal{M}$,
which concludes the proof.
\end{proof}

The remaining question is under which conditions the model satisfies assumptions $(ii)$ and $(iii)$ of the theorem.
In the following, we provide an intuition together with a practical heuristic used in our experiments,
while leaving a rigorous formal investigation to future work.
Specifically, we rely on out-of-distribution images as a source of anomalous patterns.
The main intuition is that the resulting perturbations are predominantly orthogonal
to the manifold of normal data.

\paragraph{Corruption model.}
Let $p(\bfx)$ denote the distribution of normal data supported on a differentiable manifold
$\mathcal{M}\subset\mathbb{R}^d$.
A partially corrupted observation is generated as
\begin{equation}
\hat{\bfx} = \bfx + \alpha\,\boldsymbol{\delta},
\qquad \bfx\sim p,\quad \alpha\sim\pi,
\qquad \boldsymbol{\delta}_{N(\hat{\bfx})}=\mathbf{0},
\label{eq:corr_model}
\end{equation}
where $\pi$ is a transparency distribution on $(0,1]$.

We assume that the corruption vector is predominantly normal to the manifold. That is,
there exists $\eta\in[0,1)$ such that, for typical pairs $(\bfx,\boldsymbol{\delta})$,
\begin{equation}
\big\|\Pi_{T_{\bfx}\mathcal{M}}(\boldsymbol{\delta})\big\|
\le
\eta\,
\big\|\Pi_{N_{\bfx}\mathcal{M}}(\boldsymbol{\delta})\big\|,
\label{eq:ND}
\end{equation}
where $\Pi_{T_{\bfx}\mathcal{M}}$ and $\Pi_{N_{\bfx}\mathcal{M}}$ denote the orthogonal projections onto the tangent and normal subspaces of the manifold at $\bfx$, respectively.
In practice, we approximate this condition by setting
$\boldsymbol{\delta}_{A(\hat{\bfx})} = \bfy_{A(\hat{\bfx})} - \bfx_{A(\hat{\bfx})}$
for some out-of-distribution image $\bfy \notin \mathcal{M}$ that is sufficiently distant from the manifold.
Informally, this corresponds to using images that look noticeably different from normal examples.
Furthermore, we assume that $\pi$ assigns nonzero probability to arbitrarily small values of $\alpha$.

\paragraph{Compatible normal sources.}
For a partially corrupted input $\hat{\bfx}$ we define the set of normal samples compatible with $\hat{\bfx}$ as
\begin{equation}
\text{CNS}(\hat{\bfx})
=
\Big\{
\bfx\in\mathcal{M}\;:\;
\bfx_{N(\hat{\bfx})}=\hat{\bfx}_{N(\hat{\bfx})}
\ \text{and}\ 
\exists\,\alpha\in(0,1],\,\boldsymbol{\delta}\ \text{s.t.}\ 
\hat{\bfx}=\bfx+\alpha\boldsymbol{\delta}
\Big\}.
\label{eq:defC}
\end{equation}
We measure ambiguity on the corrupted region by the diameter
\begin{equation}
\mathrm{diam}\big(\text{CNS}(\hat{\bfx})\big)
:=
\sup_{\bfx,\bfx'\in\text{CNS}(\hat{\bfx})}
\|\bfx_{A(\hat{\bfx})}-\bfx'_{A(\hat{\bfx})}\|.
\end{equation}

We now show how our corruption process induces the contraction toward the manifold that ensures properties $(ii)$ and $(iii)$ in Theorem~\ref{T_08071253}.
\begin{proposition}
\label{P_09071047}
Consider the following unconstrained optimization problem:
\begin{equation}
\label{op_24041055}
\begin{aligned}
& \underset{\bftheta \in \mathbb{R}^k}{\text{minimize}}
& & 
\mathbb{E}_{q(\hat{\bfx})}\!\left[
\mathbb{E}_{p(\bfx)}\!\left[
\|f_{\bftheta}(\hat{\bfx})-\bfx\|^2
\;\middle|\;
\mathcal{C}(\bfx) = \hat{\bfx}
\right]
\right],
\end{aligned}
\end{equation}
under the corruption model \eqref{eq:corr_model}--\eqref{eq:ND} denoted by $\mathcal{C}$.
Let $f^*$ denote a Bayes-optimal solution of the above optimization problem.
Then there exists a constant $C \in (0,1)$ such that, for all admissible corrupted samples
$\hat{\bfx}$,
\begin{equation}
\label{E_09071118}
\mathrm{diam}\big(\mathrm{CNS}(f^*(\hat{\bfx}))\big)
<
C \cdot \mathrm{diam}\big(\mathrm{CNS}(\hat{\bfx})\big).
\end{equation}
\end{proposition}

\begin{proof}
The assumption about the perturbations with normal-dominant components implies that locally the set
$\mathrm{CNS}(\hat{\bfx})$ is confined within a cone around the normal direction of
$\mathcal{M}$ in the corrupted coordinates of $\boldsymbol{\delta}$.
By Theorem~\ref{T_23040830}, the Bayes-optimal mapping produces the conditional mean of the compatible normal sources, which lies near the center of this cone and therefore yields a
\emph{blurry} single-step completion.
Because the training distribution contains transparent perturbations, i.e., arbitrarily small values of $\alpha$,
the corrected sample $\hat{\bfx}^{+} := f^*(\hat{\bfx})$ can only be generated from normal samples in a strictly smaller neighborhood of
$\mathcal{M}$ than $\hat{\bfx}$.

The corresponding reduction factor is controlled by the aperture of the local cone of compatible normal sources.
This aperture is determined by the maximal deviation from the normal-dominance condition in \eqref{eq:ND}.
Hence, under a uniform normal-dominance margin, there exists a constant $C\in(0,1)$ such that
equation (\ref{E_09071118}) holds for all admissible corrupted samples $\hat{\bfx}$.
This idea is illustrated in Figure~\ref{fig_proof}.
\end{proof}
\begin{figure}[t]
\centering
\includegraphics[scale = 0.75]{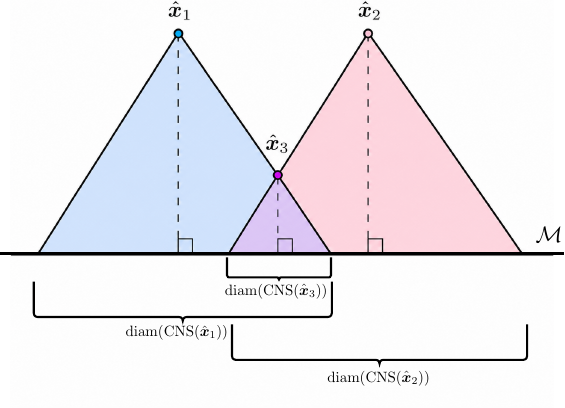}
\caption{
Illustration of the progressive reduction in the diameter of compatible normal sources under iterative application of the model.
The manifold of normal data $\mathcal{M}$ is represented schematically by the horizontal black line.
The leftmost triangle (blue) indicates the cone of compatible normal sources for the corrupted sample $\hat{\bfx}_1$, while the rightmost triangle (red) indicates the corresponding cone for $\hat{\bfx}_2$.
The small triangle (lilac) illustrates that, as the corrected sample moves closer to $\mathcal{M}$, the set of compatible normal sources becomes more restricted, resulting in a smaller diameter.
}
\label{fig_proof}
\end{figure}

Together, Theorem~\ref{T_23040830} and Proposition~\ref{P_09071047} provide the geometric mechanism underlying assumptions $(ii)$ and $(iii)$ of Theorem~\ref{T_08071253}.
Theorem~\ref{T_23040830} states that a Bayes-optimal solution $f^*$ maps each corrupted input to the conditional mean of the compatible normal sources determined by the corruption process $\mathcal{C}$.
Under the normal-dominance assumption, these compatible sources are confined to a cone around the normal direction of the manifold.
Proposition~\ref{P_09071047} shows that, after applying $f^*$, the diameter of this cone decreases by a uniform multiplicative factor.

Since the compatible normal sources are defined by fixing the normal coordinates and varying only the corrupted coordinates, their diameter measures the remaining ambiguity along the corresponding manifold slice.
Geometrically, this diameter corresponds to the local chord length between compatible normal completions on the same slice as illustrated in a cartoon in Fig.~\ref{fig_cns_chord}.
For a smooth local manifold slice with bounded curvature, the deviation of the conditional mean from the slice is controlled by this chord length.
\begin{figure}[b]
\centering
\includegraphics[width=0.5\linewidth]{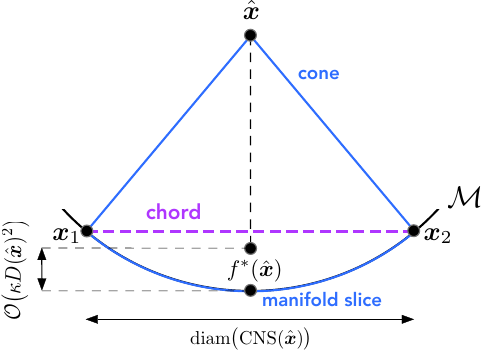}
\caption{
Geometric illustration of the compatible-source diameter for a corrupted sample $\hat{\bfx}$.
The compatible normal sources lie on the same local manifold slice (blue) and span a chord (purple), whose length corresponds to the diameter $\mathrm{diam}(\mathrm{CNS}(\hat{\bfx}))$.
The Bayes-optimal output $f^*(\hat{\bfx})$ is the conditional mean of these compatible sources and therefore lies off the curved slice.
For a smooth slice with local curvature bounded by $\kappa$, this off-manifold deviation is of order $\mathcal{O}(\kappa D(\hat{\bfx})^2)$.
Thus, reducing the compatible-source diameter also reduces the distance between the Bayes-optimal output and the manifold slice.
}
\label{fig_cns_chord}
\end{figure}
More precisely, if
$D(\hat{\bfx}) := \mathrm{diam}\bigl(\mathrm{CNS}(\hat{\bfx})\bigr)$
and $\kappa$ denotes an upper bound on the local curvature of the corresponding manifold slice, then the deviation of the conditional mean from the slice is of order
\[
\operatorname{dist}\bigl(f^*_S(\hat{\bfx}),\mathcal{M}_S(\hat{\bfx})\bigr)
=
\mathcal{O}\!\left(\kappa D(\hat{\bfx})^2\right).
\]
Thus, the deviation caused by averaging over a curved slice decreases quadratically with the compatible-source diameter.
By Proposition~\ref{P_09071047}, the compatible-source diameter satisfies
\[
D(f^*(\hat{\bfx})) \leq C D(\hat{\bfx}),
\qquad 0<C<1.
\]
Consequently, under uniform curvature control, the curvature-induced deviation decreases, up to curvature-dependent constants, by a factor of order $C^2$.

The effective contraction factor may vary with the local curvature and geometry of the slice.
However, under uniform curvature control and a uniform reduction of compatible-source diameter, these local contraction factors are bounded above by a common constant $\rho < 1$.
This explains how the multiplicative reduction of the compatible-source diameter induces the slice-contraction behavior required in assumption $(ii)$ of Theorem~\ref{T_08071253}.

Moreover, because the Bayes-optimal correction is the conditional mean of the compatible normal sources, the magnitude of the remaining correction is controlled by the same local ambiguity.
As the compatible-source diameter decreases, the conditional mean changes by a decreasing amount, and the correction magnitude vanishes near the corresponding manifold slice.
This yields the vanishing-correction behavior required in assumption $(iii)$ of Theorem~\ref{T_08071253}.
Thus, the decreasing diameter of the compatible normal sources explains why iterative application of $f^*$ moves the sample progressively closer to $\mathcal{M}$ while producing corrections of decreasing magnitude.

\paragraph{Role of transparency and corruption geometry.}
Training with transparent corruptions, i.e., a continuum of perturbation magnitudes, exposes the model to inputs
arbitrarily close to the normal-data manifold and therefore determines the local behavior of the learned mapping
in a neighborhood of $\mathcal{M}$.
However, transparency alone is not sufficient to guarantee contractive behavior.
A contraction bias is obtained only if the corruption directions that appear close to the manifold are predominantly
normal to $\mathcal{M}$.
In this case, small-amplitude corruptions correspond to displacements away from the manifold, and minimizing the
conditional reconstruction objective enforces corrections that reduce these displacements.
Conversely, if small-amplitude corruptions are dominated by tangent components, they correspond to valid variations
along the manifold and should not be removed.

Thus, transparency restricts the perturbation directions that influence the local Jacobian of the learned mapping:
near $\mathcal{M}$, the model is exposed only to corruption patterns whose normal component dominates.
Under this condition, the learned operator is biased toward contractive behavior in the corrupted coordinates.

\paragraph{Generalization beyond artificial corruptions.}
Theorem~\ref{T_08071253} guarantees that iterative application of the model converges to a fixed point on the normal-data manifold. However, an important subtlety is that this convergence is supported by an emerging contractive behavior toward the manifold. In this sense, the model does not merely correct a particular corrupted input, but contracts the ambient space toward the manifold by becoming constant along (possibly nonlinear) trajectories leading back to normal data.
We argue that the model adopts a filtering strategy: it progressively suppresses anomalous patterns, inpaints the missing content from the surrounding spatial context of normal regions, and finally stitches the reconstructed regions together. Figure~\ref{fig_layers} illustrates this behavior through the activation patterns of a trained network on two representative examples.
Since this strategy ignores the corruption pattern during inpainting, the model becomes constant along off-manifold trajectories toward the data manifold, substantially reducing the effective problem complexity and thereby mitigating the curse of dimensionality.
\begin{figure}[t]
\centering
\includegraphics[scale = 0.64]{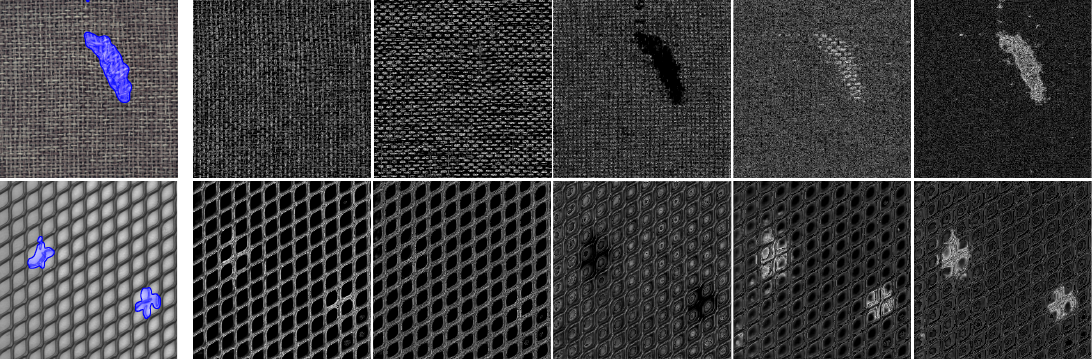}
\caption{
Illustration of activation patterns in individual layers of a trained model for two examples from the MVTec AD dataset, one from the carpet category and one from the grid category.
In each row, the first image shows the input image, with the anomalous region marked in blue. The remaining images show activations of selected feature maps from a layer close to the final output layer.
The selected feature maps support the intuition that the network corrects anomalous regions through a filtering strategy: suppressing anomalous patterns, inpainting the missing content, and stitching the reconstructed regions into the surrounding normal structure.
}
\label{fig_layers}
\end{figure}
By construction, a fully opaque corruption removes the information in the corrupted region and replaces it with a pattern that is independent of the reconstruction target. Hence, the corruption appearance contains no reliable information about the missing normal content.
A solution that relies on the corruption pattern can therefore only do so through memorization of observed corruption--reconstruction associations. Such a solution is highly unstable and unlikely to be favoured by stochastic gradient-based optimization.
A corruption-invariant solution instead ignores the corruption appearance and providing a very stable solution with respect to changes in the corruption pattern.
Contraction in off-manifold directions follows directly from this invariance. If the output does not depend on the corruption pattern, then different corrupted inputs with the same normal context are mapped to the same or similar reconstructions.
By Theorem \ref{T_23040830}, the optimal solution is independent of the corruption pattern. Thus, corruption-invariant correction is both the stable solution and the global optimum of the learning problem.

\paragraph{On limitations for logical anomaly detection.}
While the proposed formulation provides a principled framework for SAD, where anomalies manifest as deviations from locally predictable appearance structure, it cannot fully address logical anomalies. Projection-based models implicitly approximate a manifold of normal appearance statistics, $\mathcal{M}_{\mathrm{app}}$, embedded in the broader manifold of natural images, $\mathcal{M}_{\mathrm{nat}}$. Logical anomalies, however, may violate higher-level constraints such as object counts, spatial relations, or semantic consistency while still preserving locally plausible appearance statistics. Such constraints can therefore be understood as defining a further restriction $\mathcal{M}_{\mathrm{logic}} \subset \mathcal{M}_{\mathrm{app}}$: an image may lie outside $\mathcal{M}_{\mathrm{logic}}$ while remaining within $\mathcal{M}_{\mathrm{app}}$. Consequently, logical anomalies need not induce an off-manifold deviation with respect to the learned appearance manifold and may produce little or no projection residual.
Future work will investigate the integration of projection-based appearance modeling with structured modeling and graph-based approaches~\cite{math11122628, BauerSM17, Bauer2019, OptMPMP}.

\subsection{Regularization based on Jacobian Penalties}
\label{subsec:jacobian}
\noindent Geometrically, our autoencoder $f$ is designed to project corrupted inputs 
back onto the manifold of normal samples 
$\mathcal{M}$ while acting as the identity on $\mathcal{M}$.
Together, these requirements suggest that $f$ should be idempotent, i.e.\ $f(f(\hat{\bfx})) = f(\hat{\bfx})$.
Instead of adding an explicit penalty to the objective we instead propose
to constrain the Jacobian by adding the term $\|J_f(\hat{\bfx}) - I_d\|_F$ to the training objective (\ref{eq_21041620}),
where $J_f(\hat{\bfx})$ and $I_d$ denote the Jacobian of the model and the identity mapping, respectively.

To analyze how the regularizer promotes idempotency, we expand $f$ at $\bfx \in \mathcal{M}$ and evaluate it at $f(\hat{\bfx})$,
where $\hat{\bfx}$ is a corrupted version of $\bfx$.
We note that identity behavior is not enforced uniformly, but arises predominantly along tangent directions
of the manifold due to its interaction with the reconstruction error, which in turn induces contraction toward the manifold.
For the sake of argument, we consider the extreme case $J_{\bftheta}(\bfx) = I_d$, yielding
\begin{equation*}
\begin{aligned}
f(f(\hat{\bfx})) 
&= f(\bfx) + J_{\bftheta}(\bfx)\bigl(f(\hat{\bfx}) - \bfx\bigr) + R_2 \\
&= f(\hat{\bfx}) + (f(\bfx) - \bfx) + R_2,
\end{aligned}
\end{equation*}
where $R_2$ is the second-order Taylor remainder.
If the Jacobian is $L$-Lipschitz along the segment from $\bfx$ to $f(\hat{\bfx})$, then
$\|R_2\| \le L\|f(\hat{\bfx})-\bfx\|^2$.
Hence,
\begin{equation}
\label{E_22092221}
\underbrace{f(f(\hat{\bfx})) - f(\hat{\bfx})}_{\text{idempotency residual}}
=
\underbrace{f(\bfx) - \bfx}_{\text{bias on } \mathcal{M}}
+
\mathcal{O}\!\big(
\underbrace{\|f(\hat{\bfx})-\bfx\|^2}_{\text{reconstruction error}}
\big).
\end{equation}
Equation~(\ref{E_22092221}) thus links approximate idempotency directly to reconstruction accuracy and how well $f$ approximates the identity on the manifold.

More generally, projection operators onto $\mathcal{M}$ satisfy three equivalent structural properties:
(i) $\mathrm{Im}(f) = \mathcal{M}$,
(ii) $f(\bfx) = \bfx \Leftrightarrow \bfx \in \mathcal{M}$, and
(iii) $f$ is idempotent, i.e.\ $f(f(\bfx)) = f(\bfx)$.
Any two of these properties imply the third and therefore characterize a projection operator onto $\mathcal{M}$.

Our result in Equation~(\ref{E_22092221}) reveals that the regularization term couples these properties in the trained model:
(i) is reflected through reconstruction quality,
(ii) corresponds to the residual bias on the manifold, and
(iii) is captured by the idempotency residual.
In other words, the identity-anchored Jacobian regularization, together with a suitable corruption process
that promotes contraction toward the data manifold, locally stabilizes the model toward behavior that approximates the defining properties of a projection operator.

\begin{figure}[t]
\centering
\includegraphics[scale = 0.805]{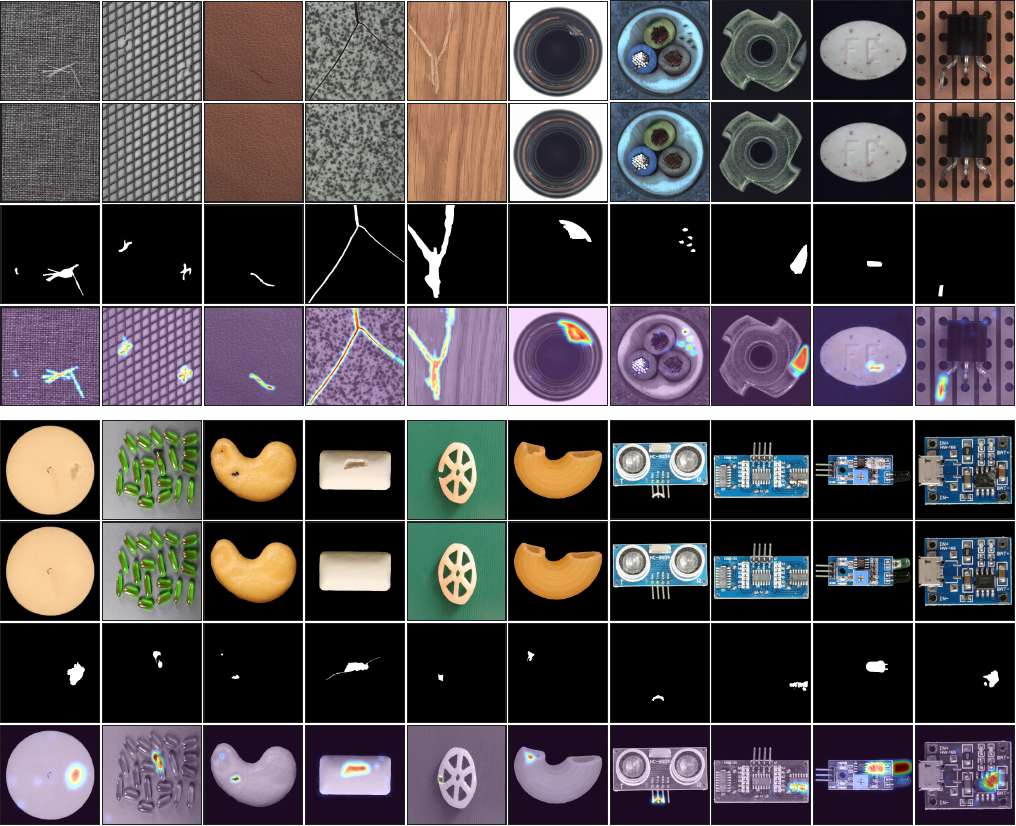}
\caption{
Representative examples of anomaly heatmaps on the MVTec~AD and VisA datasets produced by our model.
Each column shows (top to bottom) the input image, its reconstruction, the corresponding human-annotated anomaly mask for reference, and the predicted anomaly heatmap.
}
\label{fig_examples}
\end{figure}


\end{document}